\documentclass[twoside]{article}

\usepackage[accepted]{aistats2022}
\usepackage[round]{natbib}
\usepackage[utf8]{inputenc} % allow utf-8 input
\usepackage[T1]{fontenc}    % use 8-bit T1 fonts
\usepackage{hyperref}       % hyperlinks
\usepackage{url}            % simple URL typesetting
\usepackage{booktabs}       % professional-quality tables
\usepackage{amsfonts}       % blackboard math symbols
\usepackage{nicefrac}       % compact symbols for 1/2, etc.
\usepackage{microtype}      % microtypography
\usepackage{xcolor}         % colors

\usepackage{paralist}
\usepackage{graphicx}
\usepackage{dsfont}
\usepackage{amsmath}
\usepackage{amsthm}
\usepackage{algorithm}
\usepackage{algorithmic}
\usepackage{float}
\usepackage{color}
\usepackage{multicol}
\usepackage{wrapfig}
\usepackage{lipsum}
\usepackage{mwe}
\graphicspath{ {figures/} }
\usepackage{pifont}

\newcommand{\xmark}{\ding{55}}
\newtheorem{thm}{Theorem}
\newtheorem{lem}{Lemma}

\newcommand{\argmax}{\operatornamewithlimits{argmax}}

\newtheorem{remark}{Remark}

\begin{document}



\twocolumn[

\aistatstitle{Robust Stochastic Linear Contextual Bandits Under Adversarial Attacks}

\aistatsauthor{ Qin Ding \And Cho-Jui Hsieh \And  James Sharpnack}

\aistatsaddress{  Department of Statistics \\ University of California, Davis \\ qding@ucdavis.edu 
\And  Department of Computer Science \\ University of California, Los Angeles \\ chohsieh@cs.ucla.edu 
\And Amazon \footnote{} \\ Berkeley, CA \\ jsharpna@gmail.com} 
]

\begin{abstract}
  Stochastic linear contextual bandit algorithms have substantial applications in practice, such as recommender systems, online advertising, clinical trials, etc. Recent works show that optimal bandit algorithms are vulnerable to adversarial attacks and can fail completely in the presence of attacks. Existing robust bandit algorithms only work for the non-contextual setting under the attack of rewards and cannot improve the robustness in the general and popular contextual bandit environment. In addition, none of the existing methods can defend against attacked context.  In this work, we provide the first robust bandit algorithm for stochastic linear contextual bandit setting under a fully adaptive and omniscient attack with sub-linear regret. Our algorithm not only works under the attack of rewards, but also under attacked context. Moreover, it does not need any information about the attack budget or the particular form of the attack. We provide theoretical guarantees for our proposed algorithm and show by experiments that our proposed algorithm improves the robustness against various kinds of popular attacks.
\end{abstract}

\section{INTRODUCTION}\label{intro_robust}

% {\color{blue}Introduce bandit problem. }
A stochastic linear contextual bandit problem models a repeated game between a player and an environment. At each round of the game, the player is given the information of $K$ arms, usually represented by $d$-dimensional feature vectors, and the player needs to make a decision by pulling an arm based on past observations. Only the reward associated with the pulled arm is revealed to the player and the relationship between the rewards and features follows a linear model. The goal of the player is to maximize the cumulative reward or minimize the cumulative regret over $T$ rounds. Over the past few decades, bandit algorithms have enjoyed popularity in substantial real-world applications, such as recommender system \citep{li2010contextual}, online advertising \citep{schwartz2017customer} and clinical trials \citep{woodroofe1979one}.

% {\color{blue}Bandit algorithms are not robust under adversarial attacks. Attacks are designed for both multi-arm bandits [xx] and contextual bandit [xx]. Example. LinUCB LinTS fail completely under attack.  }

\footnotetext{Work done prior to joining Amazon}

Classic stochastic linear contextual bandit algorithms such as Linear Upper Confidence Bound (LinUCB) \citep{li2010contextual} and Linear Thompson Sampling (LinTS) \citep{agrawal2013thompson} algorithms can achieve an optimal regret upper bound $\tilde O(d\sqrt{T})$, where $\tilde O$ hides any poly-logarithm factors. Both LinUCB and LinTS have regret upper bounds that match the lower bound $O(d\sqrt{T})$ up to logarithm factors in infinite-arm problems. 
In recent years, adversarial attacks have been extensively studied in order to understand the robustness of machine learning algorithms, including bandit algorithms. Attacks were designed for both the contextual bandit problem and the multi-armed bandit (MAB) problem. The types of attacks include both the attacks on rewards and context. 
For example, in MAB setting, \cite{jun2018adversarial} proposes an oracle attack to modify the rewards of the pulled arms to be $\epsilon$ worse than the target arm. In contextual bandit setting, \cite{garcelon2020adversarial} proposes to attack rewards by converting the reward of pulled arms into random noise. For attack of context, \cite{garcelon2020adversarial} proposes to dilate the features of the pulled arms if they are not among the target arms of the attacker. Due to the optimality of LinUCB and LinTS algorithms, they will stop pulling the sub-optimal arms eventually. However, under adversarial attacks, the real optimal arm appears sub-optimal to the algorithm, which makes classic algorithms fail completely under attacks \citep{garcelon2020adversarial,jun2018adversarial,liu2019data,ma2018data}.

Demonstrated by the aforementioned works, adversarial attacks have become major concerns before applying bandit algorithms in practice. For example, in recommender systems, adversarial attacks will trick the algorithm into imperfect recommendations and lead to disappointing user experiences; In online advertising, attacks can be made by triggering some ads and not clicking on these ads, which fools the system to think these ads have low click-through-rate and therefore benefit its competitors \citep{lykouris2018stochastic}. Thus, it is urgent to design a robust bandit algorithm that works well under attacks.
% {\color{red}(Cho: also, since our emphasis is that we are the first defense algorithm for contextual setting, )}

% {\color{blue} Existing defense on multi-arm bandit, some recent work consider linear bandit setting (no contextual). Difficulty of extending to the contextual setting (use PE as example).  }

Lots of efforts have been made to design robust bandit algorithms under the attack of rewards with a fixed attack budget $C$. However, all of them are proposed only for the non-contextual setting. 
\cite{gupta2019better,lykouris2018stochastic} established nearly optimal algorithms for the corrupted multi-armed bandit (MAB) problem. % , which is a special case of linear bandit problems. 
Subsequently, \cite{bogunovic2020corruption} generalizes the idea of \cite{lykouris2018stochastic} to the corruption-tolerant Gaussian process bandit problems. Recently, a robust phase elimination (PE) algorithm \citep{bogunovic2020stochastic} is derived for a fixed finite-arm stochastic linear bandit problem under attacks. However, PE algorithm needs to pull a fixed feature many times in each phase, %which requires that the features of arms do not change over time,
implying that it does not work for (or at least non-trivial to be extended to) the contextual bandit setting. As suggested by \cite{bogunovic2020stochastic}, the introduction of context significantly complicates the problem and only under a ``diversity context'' assumption \citep{bogunovic2020stochastic}, a Greedy algorithm can improve the robustness. Adversarial corruptions in linear contextual bandits was considered before by \cite{kapoor2019corruption}, however, \cite{kapoor2019corruption} can only achieve $O(CT)$ regret which is trivial in bandit environment. 
To the best of our knowledge, there is no existing work that can achieve sub-linear regret and improve the robustness of contextual bandit algorithms under attacks of rewards in the general environment. 
Finally, we find that all the previous works only consider the case of corrupted rewards and cannot deal with fully adaptive attacked context, which is also a common type of attack in contextual bandit  \citep{garcelon2020adversarial}.

% none of these algorithms work under contextual bandit environment, where the features of arms can change over time. As mentioned by \cite{bogunovic2020stochastic}, the introduction of context significantly complicates the problem. To the best of our knowledge, there is no existing work that improves the robustness of stochastic linear contextual bandit algorithms in the general setting, which is a more popular and general bandit model in practice. Finally, we find that all the previous works only consider the case of corrupted rewards and cannot deal with attacked context, which is also a common type of attack in linear bandit environment \cite{garcelon2020adversarial,ma2018data}. {\color{red}(Cho: I feel we can shorten this paragraph. Several sentences repeated mention that previous methods cannot deal with attacked context. Also, is our algorithm good also under contextual bandit setting with perturbed reward? We should make our contribution more clear, maybe think about summarizing the contribution in the intro?)}
% {\color{blue} Our method. }

In this work, we consider the stochastic linear contextual bandit algorithm under unknown adversarial attacks with a fixed  attack budget, where the feature vectors of arms are contextual and can change over time. In addition, the attackers can be omniscient and fully adaptive, in the sense that the attacker can adjust the attack strategy after observing the player's decisions at each round. We first show that when the total attack budget $C$ is known, a simple modification of the classic LinUCB algorithm with an enlarged exploration parameter can obtain a regret upper bound $\tilde O(d(C+1)\sqrt{T})$. In contrast, the traditional LinUCB algorithm suffers from linear regret when the attack budget $C\sim O(\log T)$ \cite{garcelon2020adversarial}.  
% {\color{blue}(Cho: maybe we can also show what's the regret of LinUCB under attack? Is it linear? )} 
When the total attack budget is unknown, we should learn $C$ online. While learning hyper-parameters online is difficult in bandit setting, the bandit-over-bandit (BOB) algorithm \citep{cheung2019learning} in the non-stationary bandit environment successfully adjusts its sliding-window size online to adapt to changing models. Motivated by the BOB algorithm, we design a similar two-layer bandit algorithm called RobustBandit to deal with unknown attack budget. The high-level idea is to use the top layer to select an appropriate candidate attack budget in the adversarial environment and follow the classic LinUCB policy with increased exploration parameter based on the selected budget in the bottom layer. We show that with the two-layer design, our proposed RobustBandit algorithm can achieve $\tilde O(d^{\frac{1}{2}} (C+1) T^{\frac{3}{4}})  + \tilde O( d^{\frac{3}{2}} C T^{\frac{1}{4}} )$ regret upper bound under unknown adversarial attacks. We also show in experiments that a simple modification of our RobustBandit algorithm achieves substantial improvements in robustness. Our algorithm works no matter the attack is for rewards or context and it is simple to implement. We summarize our contributions below:
\begin{enumerate}
    \item We propose the RobustBandit algorithm, which is the first algorithm that can achieve sub-linear regret and improve the robustness of linear contextual bandit problems with no additional assumptions, such as the ``diversity context'' assumption, weak adversary assumption, etc. 
    \item Our proposed algorithm works for infinite-arm problems under the attack of rewards. It is also the first work that improves the robustness under the fully adaptive attack of context.
    \item 
    Our algorithm does not need to know % anything about the attack, such as 
    the algorithm used for attack or the attack budget. Furthermore, the method does not need to know whether the attack is for rewards or context under the finite-arm problem. % {\color{blue}(cho: Try to make this more clear. Currently the first sentence is not precise since we need to know whether it is conext perturbation of reward perturbation. )}
    % \item \textcolor{red}{Our algorithm does not need to know anything about the attack, such as the attack budget or whether the attack is for rewards or context under the finite-arm problem.}
\end{enumerate}
% Finally, we address that our proposed algorithm is the first work that improves the robustness of contextual bandit problems under attacks of both rewards and context. Moreover, it does not need to know whether the attack is for rewards or context when making decisions.

\textbf{Notations:} We use $\theta$ to denote the true model parameter. 
For a vector $x \in \mathbb{R}^d$, we use $\|x\|$ to denote its $l_2$ norm and $\|x\|_A = \sqrt{x^T A x}$ to denote its weighted $l_2$ norm associated with a positive-definite matrix $A \in \mathbb{R}^{d\times d}$. Denote $[n] := \{1,2,\dots, n\}$ and $\lceil b \rceil$ as the minimum integer such that $\lceil b\rceil \geq b$.

\section{RELATED WORK}\label{rw_robust}
In this section, we will focus on discussing previous works on improving the robustness of bandit algorithms. Note that very few works consider the attacks of context. \cite{yang2021robust} considers imperfect context and the attack of context cannot be adjusted by the attacker adaptively at each round based on previous observations, which makes the problem much easier. In the following, all the previous works discussed only consider the attacks of rewards. 
% {\color{red}(Maybe we need to talk more about attacks here. 1) which attacks are in the contextual bandit setting and which are in the multi-arm bandit setting? 2) Since one major contribution is a defense algorithm against perturbed context, maybe mention which attacks are trying to perturb context? )}

In case of multi-armed bandit (MAB) problems under adversarial corruptions, \cite{lykouris2018stochastic} designs a multi-layer active arm elimination race algorithm (MLAER) when the total attack budget is unknown. It obtains multiplicative regret upper bound $O(\frac{KC + \log T}{\Delta} \log (KT))$, where $\Delta$ is the mean reward gap between the optimal and sub-optimal arms.  \cite{gupta2019better} improves the robustness of MAB problems by a novel restarting algorithm, BARBAR, and derives a near-optimal regret upper bound $O(KC+\frac{(\log T)^2}{\Delta} )$, where the regret depends on $C$ linearly. The key is to ensure the algorithm never permanently eliminates an arm and so the seemingly-not-so-good arms always get a small amount of resource.
Both MLAER and BARBAR only work for a weaker adversary that cannot adjust the attack strategy based on the player's decisions. This limits their applications when the attacks are fully adaptive, which is a popular assumption when designing bandit attacks \citep{garcelon2020adversarial,jun2018adversarial,liu2019data,ma2018data}.

% {\color{red}(Cho: so is this still multi-arim bandit setting?)} 
In the linear bandit setting with non-contextual features, \cite{bogunovic2020corruption} considers the Gaussian process bandit problems with a function of bounded RKHS norm, and extends the idea in MLAER to Fast-Slow GP-UCB algorithm with a $\tilde O(d(C+1)\sqrt{T})$ regret upper bound.
Note that \cite{bogunovic2020corruption} assumes that the attacks cannot be fully adaptive, which makes it easier to defend against the corruptions. Built upon \cite{gupta2019better}, \cite{li2019stochastic} proposes support bias exploration algorithm (SBE) for stochastic linear bandits and derives a $O(\frac{d^{\frac{5}{2}} C\log T}{\Delta}+ \frac{d^6(\log T)^2}{\Delta^2} )$ regret upper bound. %  for stochastic linear bandits, where the regret depends on $\Delta$, the mean reward gap between the optimal and sub-optimal arms. 
Most recently, \cite{bogunovic2020stochastic} proposes a phase elimination algorithm (PE) to improve the robustness 
of linear bandit and obtains $\tilde O(\sqrt{dT} + C^2 + Cd^\frac{3}{2} \log T)$ regret upper bound. In each phase, the PE algorithm computes a design over the set of potential optimal arms and plays each arm in this set with number of times proportional to the computed design. In the meantime, it requires that every arm in its support is played at least some minimal number of times. Since PE needs to play an arm with a fixed feature many times, it does not work in the contextual bandit environment, where the features of arms are contextual and can change over time, which is a more general and popular problem in practice.

% We note that none of these works can deal with the contextual bandit, which is a more common in real applications.
For the contextual bandit problem under attacks of rewards,
% {\color{red}(Cho: mention that we start to discuss previous work on robust contextual bandit?)}
\cite{bogunovic2020stochastic} shows that under a ``diversity context'' assumption \citep{bogunovic2020stochastic,ding2021efficient,wu2020stochastic}, 
% {\color{blue}(cho: can we use one or two sentences to explain what is the diversity context assumption? )}, 
a simple Greedy algorithm can obtain $\tilde O\left (\frac{1}{\lambda_0} (\sqrt{dT} + C\log T) \right)$ regret upper bound in the presence of attacks. Here, $\lambda_0$ is the degree of the diversity, which is defined as the minimum eigenvalue of the covariance matrix of the feature vectors lying in any half space. However, due to the partial feedback setting, contextual feature vectors of the pulled arms are highly biased and so the ``diversity context'' assumption rarely holds in bandit environment. When ``diversity context'' assumption does not hold, $\lambda_0 =0$ and this makes the regret bound of the Greedy algorithm meaningless. Moreover, under similar diversity assumptions \citep{wu2020stochastic}, LinUCB was shown to obtain a constant regret, which makes the optimality of the Greedy algorithm under attacks questionable. Concurrently, \cite{zhao2021linear} studies the linear contextual bandit problem under corruption on rewards and proposes Multi-level weighted OFUL algorithm which achieves $\tilde O( C^2 d \sqrt{\sum_{t=1}^T \sigma_t^2} + C^2\sqrt{dT})$ regrets, where $\sigma_t$ is the variance of observed reward at round $t$. 

Finally, we emphasize that there is no existing work that considers contextual bandit problems with fully adaptive attacked context, even under the ``diversity context'' assumption. We summarize the properties of different algorithms in Table \ref{compare_algo}. 
% {\color{red}(We should also mention in the table which algorithm is contextual bandit and which is multi-arim bandit? also emphasize diverse assumption in the table?)}

\begin{table*}[h]
  \caption{Comparison of different algorithms with unknown total attack budget. Column ``Fully adaptive'' means whether the algorithm works when the attacker is fully adaptive and can adapt his attack strategy based on the player's decisions at the current round.}
  \label{compare_algo}
  \centering
  \begin{tabular}{ccccc}
    \toprule
    Algorithm     & Contextual     & Infinite-arm    & Fully adaptive  & Attacked Context   \\ % & Regret Upper Bound \\
    \midrule
    MLAER \citep{lykouris2018stochastic} \textbf{(MAB only)}  & \xmark  & \xmark & \xmark  & \xmark \\ % & $O(\frac{KC + \log T}{\Delta} \log (KT))$     \\
    BARBAR \citep{gupta2019better} \textbf{(MAB only)}   & \xmark   & \xmark & \xmark  & \xmark  \\ % & $O(KC+\frac{(\log T)^2}{\Delta} )$     \\
    SBE \citep{li2019stochastic}     & \xmark  & \checkmark & \xmark & \xmark  \\ % & $O(\frac{d^{\frac{5}{2}} C\log T}{\Delta}+ \frac{d^6(\log T)^2}{\Delta^2} )$     \\
    Fast-Slow GP-UCB \citep{bogunovic2020corruption}     & \xmark  & \checkmark & \xmark & \xmark  \\ % & $\tilde O(d(C+1)\sqrt{T})$     \\
    PE \citep{bogunovic2020stochastic}     & \xmark  & \xmark & \checkmark  & \xmark \\ % & $\tilde O(\sqrt{dT} + C^2 + Cd^\frac{3}{2} \log T)$     \\
    Greedy \footnote{} \citep{bogunovic2020stochastic} & \checkmark  & \checkmark & \checkmark  & \xmark \\ %  & $\tilde O(d\sqrt{T} + C\log T)$     \\
    \textbf{RobustBandit (This work)}     & \checkmark  & \checkmark & \checkmark  & \checkmark\\ % & $\tilde O(d(C+1)\sqrt{T})$      \\
    \bottomrule
  \end{tabular}
\end{table*}

\section{PROBLEM SETTING}\label{ps_robust}
We study a $K$-armed stochastic linear contextual bandit setting under corrupted observations, where $K$ can be infinite. Two types of adversarial attacks listed below are considered in this paper. 

\textbf{Attack rewards.} 
At each round $t\in [T]$, the player is given a set of context including $K$ feature vectors $\mathcal{A}_t = \{x_{t,a} | a\in [K]\} \subset \mathbb{R}^d$, where $x_{t,a}$ represents the true information of arm $a$ at round $t$. Based on previous observations and $\mathcal{A}_t$, the player pulls an arm $a_t$ at round $t$. For ease of notation, we denote $X_t:= x_{t,a_t}$. The attacker observes $a_t$ and assigns the attack $c_t$, where $c_t$ may depend on $a_t$ and other possible information. The player receives an attacked reward $\tilde Y_t = Y_t + c_t$, where
$Y_t = X_t^T \theta + \epsilon_t$ is the true stochastic reward. 
Here, $\theta\in \mathbb{R}^d$ is the unknown true model parameter and $\epsilon_t$ is a random noise. 

\textbf{Attack context.}
At each round $t$ before the player observes $\mathcal{A}_t$, the attacker calculates an attack based on past information and the true context $\mathcal{A}_t$ at the current round. The attacker then modifies the true context into the attacked context $\tilde x_{t,a}$. The attack cost paid by the attacker at this round is defined as $c_t = \sum_{a=1}^K \|\tilde x_{t,a} - x_{t,a}\|$. The player can only observe the contextual features after the attack, i.e., $\mathcal{\tilde A}_t := \{ \tilde x_{t,a} | a \in [K] \}$ and pull arm $a_t$ according to the attacked context. The observed reward for this pulled arm is the true reward $Y_t = X_t^T\theta + \epsilon_t$.

In the following, we will use $\mathcal{\tilde A}_t$, $\tilde X_t$ and $\tilde Y_t$ to denote the player's observed features set, observed features and rewards of the pulled arm at round $t$ in the attacked environment respectively. Under the attack of rewards, $\tilde X_t = X_t$ and $\mathcal{\tilde A}_t = \mathcal{A}_t$. While under the attack of context, $\tilde Y_t = Y_t$. 
We use $\mathcal{F}_t = \sigma (a_1,\dots,a_t, \mathcal{\tilde A}_1, \dots, \mathcal{\tilde A}_t, \tilde Y_1,\dots, \tilde Y_t, c_1,\dots,c_t)$ to denote the $\sigma$-algebra generated by all the information up to round $t$.
In both the above two attacks, the random noises $\epsilon_t$ are independent zero-mean $\sigma$-sub-Gaussian random variables with $\mathbb E\left[ e^{b\epsilon_t} | \mathcal{F}_{t-1} \right] \leq e^{\frac{b^2 \sigma^2}{2}}$ for all $t$ and $b \in \mathbb R$. 
Without loss of generality, we assume that $\|x_{t,a}\| \leq 1$ for all $t\in [T]$ and $a\in [K]$, $\|\theta\| \leq 1$ and the true observed reward $Y_t \in [0,1]$. We do not make any assumption on the attacker. The attacker can be fully adaptive in our model and even have complete knowledge of the bandit algorithm the player is using and the bandit environment, i.e., $\theta, \sigma, \mathcal{A}_t$, random noises $\epsilon_t$, the pulled arm $a_t$ and so on. However, to make the attacker invisible to the player, the attack strategy should obey the rules to make sure $\tilde Y_t \in [0,1]$ and $\|\tilde x_{t,a}\|\leq 1$. In addition, the attacker has an attack budget $C$ which limits the total amount of perturbations by $\sum_{t=1}^T |c_t| \leq C$. Denote $a_t^* = \argmax_{a\in [K]} x_{t,a}^T \theta$ as the optimal arm at round $t$ and $x_{t,*} = x_{t,a_t^*}$ the corresponding true feature vector. The goal of the player in both cases is to minimize the cumulative regret defined below
\begin{equation*}
    R(T) = \sum_{t=1}^T (x_{t,*} - X_t)^T \theta.
\end{equation*}

\footnotetext{The results only hold for Greedy algorithm when ``diversity context'' assumption \citep{bogunovic2020stochastic} is satisfied.}

\section{STOCHASTIC LINEAR CONTEXTUAL BANDIT UNDER KNOWN ATTACK BUDGET}
% tochastic linear contextual bandit under known attack budget}
In this section, we first show that if the player knows the attack budget $C$, then a simple modification of LinUCB algorithm with an enlarged exploration parameter performs well under corruptions.
% {\color{red}(Maybe we should briefly introduce LinUCB algorithm somewhere? )}
In the non-corrupted setting, LinUCB (see Algorithm \ref{linucb} in Appendix for more details) calculates the ridge regression (with regularization parameter $\lambda$) estimator $\hat\theta_t^0$ using all the past observations $\{(X_s, Y_s)\}_{s=1}^{t-1}$. By \cite{abbasi2011improved}, for all $x\in \mathbb{R}^d$, it 
% the true model parameter $\theta$ 
satisfies $|x^T (\hat\theta_t^0 -\theta) | \leq \beta_t \|x\|_{V_t^{-1}}$ with probability at least $1-\delta$, 
where
\begin{align}
    \beta_t & = \sigma\sqrt{d \log \left( \frac{1+t/\lambda}{\delta} \right)} + \sqrt{\lambda}  \label{beta} \\
     V_t & = \lambda I_d + \sum_{s=1}^{t-1} X_t X_t^T.
\end{align}
LinUCB policy then follows by pulling the arm with the biggest upper confidence bound with the exploration parameter $\beta_t$ as $$a_t = \argmax_a x_{t,a}^T \hat\theta_t^0 + \beta_t \|x_{t,a}\|_{{V_t}^{-1}} .$$

Under  adversarial attacks, we cannot observe the true context $X_t$ or true rewards $Y_t$. With the attacked context $\tilde X_t$ or attacked rewards $\tilde Y_t$, we can only calculate the corrupted ridge regression estimator from the observation history $\{(\tilde X_s, \tilde Y_s)\}_{s=1}^{t-1}$ and denote it as $\hat\theta_t$.
Define $\tilde V_t = \lambda I_d + \sum_{s=1}^{t-1} \tilde X_t \tilde X_t^T$. Take the attack of rewards as an example, and note that $\tilde X_t = X_t$, $\tilde Y_t = Y_t + c_t$ under attack of rewards, the difference $\hat\theta_t - \theta$ can be decomposed as in the following equation:
\begin{align}\label{1st_decomp}
    & \hat\theta_t - \theta = \tilde V_t^{-1} \sum_{s=1}^{t-1} \tilde X_s \tilde Y_s - \theta \nonumber \\
    & = \tilde V_t^{-1} \sum_{s=1}^{t-1} \tilde X_s (Y_s + c_s) - \theta \nonumber \\
    & = \hat\theta_t^0 - \theta + \tilde V_t^{-1} \sum_{s=1}^{t-1} c_s \tilde X_s.
\end{align}
Under attacked context, the analysis of $\hat\theta_t - \theta$ is more tricky since the player can only observe attacked context. However, % we still have similar results in Lemma \ref{concen_context} and Theorem \ref{context}. Proofs are deferred to appendix. Denote $\hat\theta_t^c$ the ridge regression estimator with regularization parameter $\lambda >0$ from the observation history $\{(\tilde X_s, Y_s)\}_{s=1}^{t-1}$. Define $\tilde V_t = \lambda I_d + \sum_{s=1}^{t-1}\tilde X_t \tilde X_t^T$.
by utilizing the concentration results in \cite{abbasi2011improved}, we can derive the following lemma that holds no matter the attack is on rewards or context.

\begin{lem} \label{concen_reward}
Define $\hat\theta_t$ as the solution to the ridge regression with regularization parameter $\lambda$ with observations $\{(\tilde X_t, \tilde Y_t)\}_{s=1}^{t-1}$. It satisfies the following inequality for all vector $x \in \mathbb{R}^d$ with probability at least $1-\delta$ no matter the attack is for rewards or context.
\begin{equation}\label{hp_events}
    | x^T (\hat\theta_t - \theta) | \leq (\beta_t + \gamma_t C) \|x\|_{\tilde V_t^{-1}},
\end{equation}
where $\beta_t = \sigma \sqrt{d\log \left( \frac{1+t/\lambda}{\delta}\right)} + \sqrt{\lambda}$ and $\gamma_t = \sqrt{\sum_{s=1}^{t-1} \|\tilde X_s\|^2_{\tilde V_s^{-1}} }$.
\end{lem}

% Under the well-known LinUCB policy, the player takes the action optimistically by pulling the arm having the biggest upper confidence bound, where the confidence width is defined as $\beta_t \|x_{t,a}\|_{V_t^{-1}}$ and $\beta_t$ is known as the exploration parameter. 
Based on Lemma \ref{concen_reward}, if the player knows an upper bound $C^\prime$ of the attack budget such that $C\leq C^\prime$, then we can simply enlarge the exploration parameter of LinUCB from $\beta_t$ to $\beta_t + \gamma_t C^\prime$, i.e., pull the arm
\begin{equation}\label{policy}
    a_t = \argmax_{a=1,\dots,K} \tilde x_{t,a}^T \hat\theta_t + (\beta_t + \gamma_t C^\prime) \|\tilde x_{t,a}\|_{\tilde V_t^{-1}}.
\end{equation}

By doing so, we can obtain a regret upper bound in Theorem \ref{reward} below.
\begin{thm} \label{reward}
Let $C^{\prime}$ be a constant that satisfies $C^\prime \geq C$, where $C$ is the total attack budget. Denote $R_r(T)$ and $R_c(T)$ as the cumulative regret under attack of rewards and context respectively. 
Under the attack of rewards or context, if we pull arm $a_t$ as defined in Equation \ref{policy} at round $t$, 
then with probability at least $1-\delta$ the cummulative regret satisfies
\begin{align}
    R_r(T) & \leq (2\beta_T +  \gamma_T (C+C^\prime)) \gamma_{T+1} \sqrt{T}  \label{rew_linucb} \\
    R_c(T) & \leq (2\beta_T +  \gamma_T (C+C^\prime)) \gamma_{T+1} \sqrt{T} + C  \label{context_linucb}
\end{align}
Under both cases, 
\begin{equation}
    R(T) \leq \tilde O(d\sqrt{T}) + \tilde O(d(C+C^\prime) \sqrt{T}).
\end{equation}
\end{thm}

\section{PROPOSED ALGORITHM - RobustBandit}\label{algo_robust}
\subsection{Stochastic contextual bandit under unknown attacks}
In the previous section, we have shown that when the total attack budget has a known and tight upper bound, then a simple LinUCB algorithm with an enlarged exploration parameter can improve the robustness under adversarial attacks. However, if the attack budget $C$ is agnostic, then we have to adaptively enlarge the exploration parameter to make it robust to all possible amounts of attacks. 

\begin{algorithm}[htb] 
\caption{RobustBandit Algorithm}
\label{bob}
\textbf{Input}: $T$, $\beta_t$ as in Equation \ref{beta}, epoch length $H$.
\begin{algorithmic}[1]
\STATE Initialize candidate attack budget set $J= \{0\}\cup\{2^j\}_{j=0}^{\lceil \log_2 2KT\rceil}$. 
\STATE Initialize exponential weights $w_j(1) = 1$ for $j=1,\dots,|J|$. 
\STATE Initialize the exploration parameter for EXP3 as $\alpha = \min\left\{ 1, \sqrt{\frac{|J| \log |J|}{(e-1) \lceil \frac{T}{H} \rceil}} \right\}$.
\FOR{$i = 1$ {\bfseries to} $\lceil \frac{T}{H} \rceil$}
    \STATE Initialize $\hat\theta_t = 0$, $\tilde V_t = \lambda I_d$.
    \STATE Update probability distribution for pulling candidates in $J$ as
        \begin{equation*}
            p_{j}(i) = \frac{\alpha}{|J|} + (1-\alpha) \frac{w_j(t)}{\sum_{i=1}^{|J|} w_i(t)}, \forall j=1,\dots, |J|.
        \end{equation*}
        \STATE Draw $j_i \gets j \in [|J|]$ with probability $p_j(i)$. % {\color{red}(can we just say sample $j_t$ from $[J]$ with probability $p_j(t)$? It is confusing if we have both $i$ and $j$ here.)}
    \FOR{$t = (i-1)H+1$ {\bfseries to} $\min\{T,iH\}$}
        \STATE Observe (attacked) contextual feature vectors $\tilde x_{t,a}$.
        \STATE Pull arm according to the following equation
        \begin{align*}
            a_t & = \argmax_{a=1,\dots,K} \tilde x_{t,a}^T \hat\theta_t + (\beta_t + \gamma_t J_{j_i}) \|\tilde x_{t,a}\|_{\tilde V_t^{-1}} 
        \end{align*}
        \STATE Observe reward $\tilde Y_t$. %  (attack rewards) or $Y_t$ (attack context).
    
        \STATE Update LinUCB components, i.e., $\tilde V_{t+1} = \tilde V_t + \tilde X_t \tilde X_t^T$ and $\hat\theta_{t+1} = \tilde V_{t+1}^{-1} \sum_{s=1}^t \tilde X_s \tilde Y_s$.
        \ENDFOR
        \STATE Update EXP3 components: $\hat y_t(j) \gets 0$ for all $j\neq j_i$, $\hat y_t(j) = \sum_{t=(i-1)H+1}^{iH} \tilde Y_t / p_{j} (i)$ if $j=j_i$ and % {\color{red}(replace $i_t$ by $j_t$ in the following? )}
        \begin{align*}
            % \hat y_t(j)& = \tilde Y_t / p_{j} (t) \quad \text\\ % \tag{Attack rewards} \\
            % \hat y_t(j) & = Y_t / p_{j} (t) \tag{Attack context} \\ 
            w_j(t+1) & = w_j(t) \times \text{exp} \left( \frac{\alpha}{|J|} \hat y_t(j)\right).
        \end{align*}
    
\ENDFOR
\end{algorithmic}
\end{algorithm}

From Theorem \ref{reward}, we can see that the key to design a robust % stochastic contextual bandit 
algorithm is to find an appropriate upper bound $C^\prime$ on the total attack budget $C$. Since the regret upper bound is $\tilde O(d(C+C^\prime+1)\sqrt{T})$ in Theorem \ref{reward}, a good choice of $C^\prime$ should satisfy $C\leq C^\prime \leq mC$, where $m\geq 1$ is a constant that should be as small as possible. 

The intuition behind our method is that we accumulate experience about what the best choice of $C^\prime$ is. In the meantime, we cut the whole time horizon into epochs and restart the algorithm when we need to try a different choice of $C^\prime$. The effectiveness of this restarting technique not only helps adaptively learn the best $C^\prime$, but also voids the attacker's efforts in the previous epoch. While the attack budget of the attacker is being exhausted after restarting the epochs, our algorithm will be able to learn the best arms gradually through the epochs.

Assume we restart the algorithm every $H$ rounds, then there will be $L = \lceil \frac{T}{H}\rceil$ epochs and the last epoch does not necessarily have $H$ rounds. Since both $Y_t$ and $\|x_{t, a}\|$ are upper bounded by $1$ and the attacker should make sure the corruption is undetectable, which means that $\tilde Y_t \in [0,1]$ and $\|\tilde x_{t,a}\| \leq 1$, % so for the $i$-th epoch, 
a natural upper bound on the attack budget spent by the attacker % in this epoch 
is $C\leq T$ under attack of rewards and $C\leq 2KT$ under attack of context. Assume $C_i$ is the attack spent in the $i$-th epoch by the attacker, then $C_i \leq \min(C,H)$ under attack of rewards and $C_i \leq \min(C, 2KH)$ under attack of context for all epochs $i$.

If we denote $J = \{0\}\cup\{2^j\}_{j=0}^{\lceil \log_2 2KT\rceil}$ as a set of possible choices of $C^\prime$. Then since $C_i \leq 2KH$, 
there must exist a $C^* \in J$ such that $C_i \leq C^* \leq 2C$ for all $i$, 
which makes $C^*$ a proper choice for $C^\prime$ at the every epoch. Due to the existence of adversarial corruptions, it is natural to treat the choices of $C^\prime$ as another adversarial multi-armed bandit (MAB) problem \cite{auer2002nonstochastic}. This motivates our proposed algorithm. This idea is also very similar to the Bandit-Over-Bandit idea \cite{cheung2019learning}, where the authors use an EXP3 algorithm over the LinUCB algorithm in non-stationary bandit problems. 

To be more specific, we propose a two-layer robust bandit algorithm. % Take the robust bandit under attack of rewards as an example, 
In the top layer, the algorithm uses an adversarial MAB policy, namely EXP3 \citep{auer2002nonstochastic}, to pull candidate  $J_{j_i}$ from the set $J$ in the beginning of every epoch. Here $j_i$ is the index of the pulled candidate at the $i$-th epoch and $J_j$ is the $j$-th element in the set $J$. In the $i$-th epoch, where $t = (i-1)H + 1, \dots, iH$, 
in the bottom layer, 
the algorithm enlarges the exploration parameter of LinUCB based on $C^\prime = J_{j_i}$ and pulls arm $a_t$ at round $t$ according to
\begin{equation}\label{final_policy}
    a_t = \argmax_{a=1,\dots,K} \tilde x_{t,a}^T \hat\theta_t + (\beta_t + \gamma_t J_{j_i}) \|\tilde x_{t,a}\|_{\tilde V_t^{-1}}.
\end{equation}
after observing the (attacked) context. 
After receiving the (attacked) reward of the pulled arm, this reward is fed to the bottom layers to update the components of LinUCB. When an epoch ends, we restart the LinUCB algorithm and use the rewards accumulated during the previous epoch to update the EXP3 algorithms. The EXP3 algorithm will then update a decision about which $C^\prime$ to choose in the next epoch. 
Details can be found in Algorithm \ref{bob}.  % Similar strategies can be designed for the corrupted context setting. 

We emphasize here that our algorithm does not need to know whether the attack is on rewards or context in order to make decisions if there are finite arms. Our algorithm also works for infinite-arm problem under attacks of rewards by simply setting $J = \{0\}\cup\{2^j\}_{j=0}^{\lceil \log_2 T\rceil}$ since $C\leq T$ under attack of rewards.

\subsection{Regret analysis}\label{analysis_robust}
In this section, we formally analyze the regret of our proposed Algorithm \ref{bob}. 
% We detail our analysis of the robust bandit under attack of rewards below to illustrate the general idea. More details about the regret analysis under attacked context 
Proofs are similar to \cite{cheung2019learning}. 
For a round $t$ in the $i$-th epoch, let $X_t(J_j)$ denote the true feature vector of the arm pulled at round $t$ if the chosen $C^\prime$ in the beginning of the $i$-th epoch is $J_j$ and the enlarged exploration parameter is based on $J_j$ for all rounds in this epoch. Denote $\tilde X_t(J_j)$ as the corresponding (attacked) feature vector. 
Let $C^*$ be the element in $J$ such that $C_i \leq C^* \leq 2C$.
We decompose the cumulative regret into two quantities below: % {\color{red}(cho: merge into a single line. you can use ``underbrace'' and some symbols to denote the two quantities. )}
\begin{align}
    & \mathbb{E}[R(T)] = \mathbb{E}\left [ \sum_{t=1}^T \left( {x_{t,*}}^T \theta - X_t^T \theta\right) \right ] \nonumber\\
    & = \underbrace{ \mathbb{E} \left [ \sum_{t=1}^T \left( {x_{t,*}}^T \theta - X_t (C^*)^T \theta \right) \right ]}_{\textstyle \text{Quantity (A)} }  \nonumber \\
    & + \underbrace{\mathbb{E} \left [ \sum_{t=1}^T \left( X_t (C^*)^T \theta - X_t (J_{j_i})^T \theta \right) \right ]}_{\textstyle \text{Quantity (B)} }. % \label{q2}
\end{align}
We bound the regret from Quantity (A) in the following lemma.
\begin{lem}\label{qa}
Quantity (A) is upper bounded by $\tilde O(\frac{dT}{\sqrt{H}} + dC\sqrt{H} + \frac{dCT}{\sqrt{H}} )$.
\end{lem}
\begin{proof}
Quantity (A) can be further decomposed into epochs
\begin{align}
    % \text{Quantity (A)}  =  
    \sum_{i=1}^{\lceil \frac{T}{H} \rceil} \mathbb{E} \left [ \sum_{t=(i-1)H+1}^{\min(T,iH)} \left( {x_{t,*}}^T \theta - X_t (C^*)^T \theta \right) \right ] \nonumber
\end{align}
Under attack of rewards,  
according to Theorem \ref{reward} and the choice of $C^*$, the regret of each epoch from Quantity (A) is upper bounded by $\tilde O(d\sqrt{H}) + \tilde O(d(C_i+C^*)\sqrt{H})$, where $C_i$ is the real attack budget used by the attacker in the $i$-th epoch. According to Equation \ref{rew_linucb}, \ref{context_linucb}, and summing over $\lceil \frac{T}{H}\rceil$ epochs, we get  
\begin{align}
    \text{Quantity (A)}  & \leq 
      \sum_{i=1}^{\lceil \frac{T}{H} \rceil} (2\beta_T +  \gamma_T (C_i+C^*)) \gamma_{T+1} \sqrt{H}  \nonumber\\
      & = \frac{2\beta_T\gamma_{T+1} T}{\sqrt{H}} + \gamma_T \gamma_{T+1} (C\sqrt{H} + \frac{2CT}{\sqrt{H}}) \nonumber
   % & \leq \sum_{i=1}^{\lceil \frac{T}{H} \rceil} \tilde O(d\sqrt{H}) + \tilde O(d(C_i+C^*)\sqrt{H})  \nonumber \\
\end{align}
Similarly, under attack of context, we have 
\begin{align}
    & \text{Quantity (A)}  \nonumber \\
    & \leq 
 \frac{2\beta_T\gamma_{T+1} T}{\sqrt{H}} + \gamma_T \gamma_{T+1} (C\sqrt{H} + \frac{2CT}{\sqrt{H}}) + C \nonumber
   % & \leq \sum_{i=1}^{\lceil \frac{T}{H} \rceil} \tilde O(d\sqrt{H}) + \tilde O(d(C_i+C^*)\sqrt{H})  \nonumber \\
\end{align}
We bound $\gamma_{T+1}$ in Equation \ref{bound_gamma} in Appendix, and according to Equation \ref{bound_gamma}, we have 
\begin{align}
\text{Quantity (A)} = \tilde O(\frac{dT}{\sqrt{H}} + dC\sqrt{H} + \frac{dCT}{\sqrt{H}} ). \nonumber
\end{align}
\end{proof}

The regret in Quantity (B) comes from the learning of $C^*$ using EXP3 algorithm. From the result in \cite{auer2002nonstochastic}, we have the following lemma and proof is deferred to Appendix.
\begin{lem}
\label{thm_q2}
For a random sequence of candidates $\{J_{j_1},\dots, J_{j_L}\}$ pulled by the EXP3 layer in each epoch in Algorithm \ref{bob}, where $L=\lceil T/H\rceil$, no matter the attack is for rewards or context, we have
\begin{align*}
    & \mathbb{E}\left [ \sum_{t=1}^T X_t (C^*)^T \theta\right] - E\left [ \sum_{t=1}^T X_t (J_{j_t})^T \theta \right ]  \\
    & \leq 2\sqrt{e-1} \sqrt{\lceil TH \rceil |J| \log |J|} = \tilde O(\sqrt{TH}).
\end{align*}
\end{lem}

% Under the attack of context, the bound of Quantity \ref{q2} does not depend on $C$, since the EXP3 layer proceeds solely on the 

The regret upper bound of our proposed algorithm follows directly from Lemma \ref{qa} and Lemma \ref{thm_q2}. We show the conclusion in Theorem \ref{final_thm} below. % Note that this theorem holds for both the cases of corrupted rewards and context.

\begin{thm}
\label{final_thm}
Denote $\nu= \sqrt{2d \log \frac{d\lambda + T}{d\lambda}}$. % Define $R_r(T), R_c(T)$ as the cumulative regret of Algorithm \ref{bob} under the attack of rewards and context respectively. 
For any error probability $\delta \in (0,1)$,
 Algorithm \ref{bob}  with $\beta_t = \sigma \sqrt{d\log \left( \frac{1+t/\lambda}{\delta}\right)} + \sqrt{\lambda}$,  $\alpha = \min\left\{ 1, \sqrt{\frac{|J| \log |J|}{(e-1)T}} \right\}$ and $\gamma_t = \sqrt{\sum_{s=1}^{t-1} \|\tilde X_s\|^2_{\tilde V_s^{-1}} }$ 
% and $\tilde\gamma_t = \sqrt{\sum_{s=1}^{t-1} \|\tilde X_s\|^2_{\tilde V_s^{-1}} }$, 
has the following expected total regret:
%then the expected total regret of Algorithm \ref{bob} satisfies
\begin{align*}
    \mathbb{E} [R(T)] % & 
    % \leq (2\beta_T +  \gamma_T (C+C^\prime)) \gamma_{T+1} \sqrt{T} + C + 2\sqrt{e-1} \sqrt{\lceil TH \rceil |J| \log |J|}
    % \\
    & \leq \tilde O(\frac{dT}{\sqrt{H}} + dC\sqrt{H} + \frac{dCT}{\sqrt{H}} ) + \tilde O(\sqrt{TH})
\end{align*}
\end{thm}

\begin{remark}
Since the choice of $H$ should not depend on $C$ and we want to minimize the cumulative regret with respect to $H$, we can choose $H = \tilde O(d\sqrt{T})$. The cumulative regret of our algorithm then satisfies
\begin{equation*}
    \mathbb{E} [R(T)] \leq \tilde O(d^{\frac{1}{2}} (C+1) T^{\frac{3}{4}})  + \tilde O( d^{\frac{3}{2}} C T^{\frac{1}{4}} ).
\end{equation*}
\end{remark}

\begin{remark}
In the non-corrupted setting or when the attack budget $C \leq \tilde O(1)$, the regret upper bound of our algorithm is still sub-linear. In contrast, classic algorithms such as LinUCB and LinTS fail completely under such attack budget and suffers from linear regret \cite{garcelon2020adversarial}.
\end{remark}

\begin{remark}
Our proposed algorithm is the first one in existing works that can deal with adversarial attacks under linear contextual bandit environment with sub-linear regret. Our algorithm can improve the robustness under either the attacks on rewards or context. To the best of our knowledge, there is no existing work that improves the robustness under attack of context. 
Note that our proposed algorithm does not need to know anything about the attack budget, whether the attack is for rewards or for context in order to achieve this regret upper bound for finite-arm problems. 
\end{remark}

\begin{figure*}[h]
    \centering
    
    \includegraphics[width=0.33\textwidth]{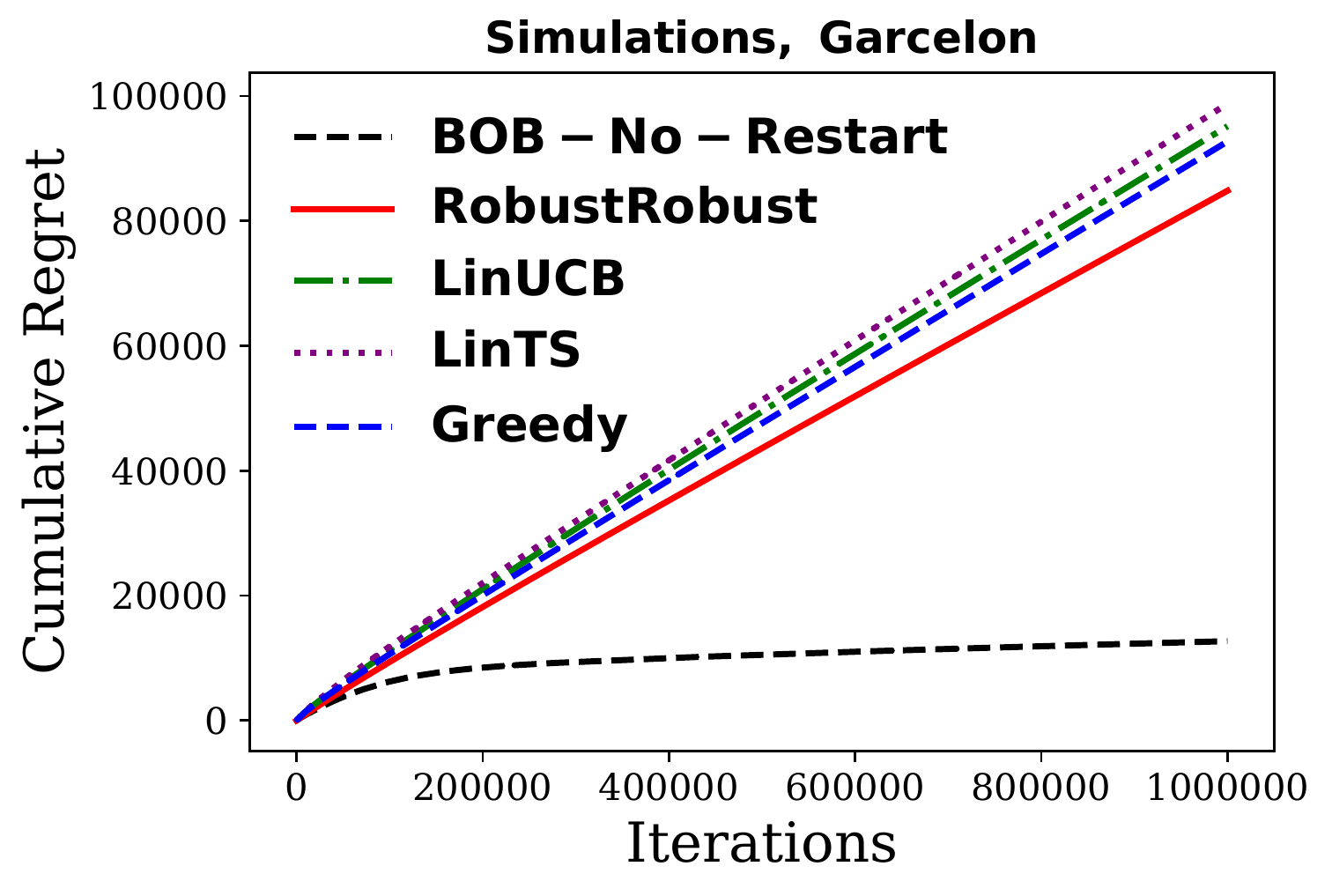}\includegraphics[width=0.33\textwidth]{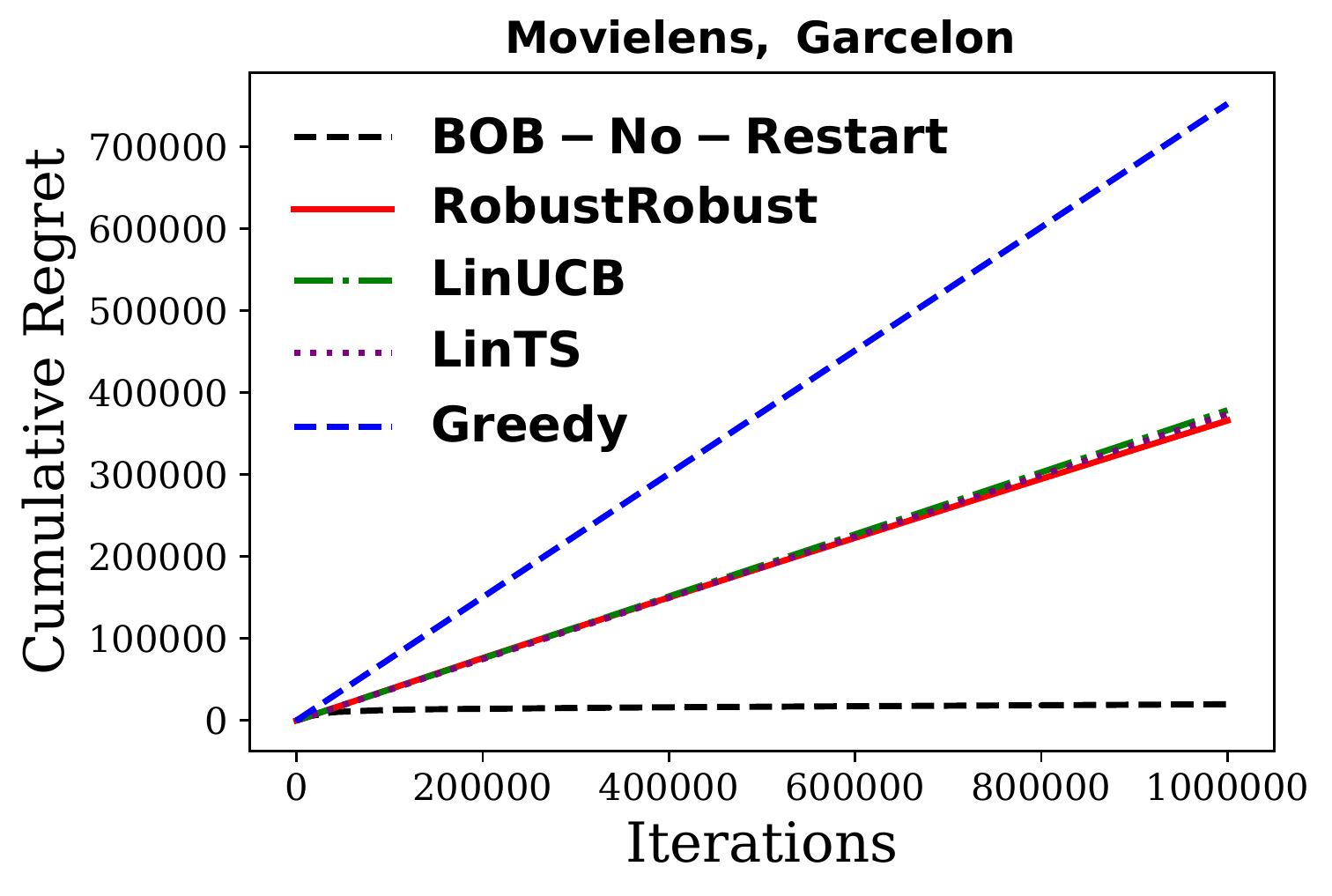}\includegraphics[width=0.33\textwidth]{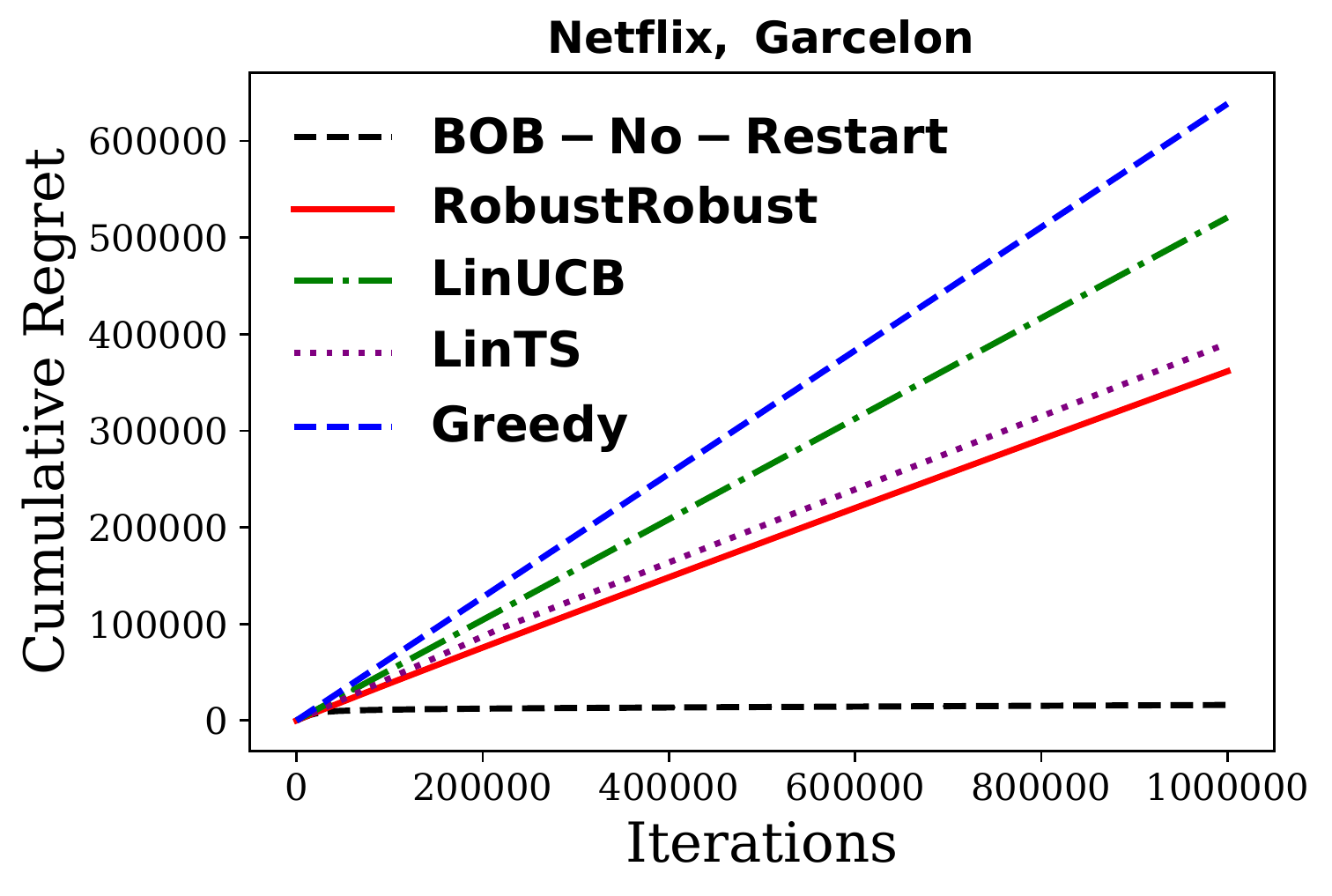}
    
    \includegraphics[width=0.33\textwidth]{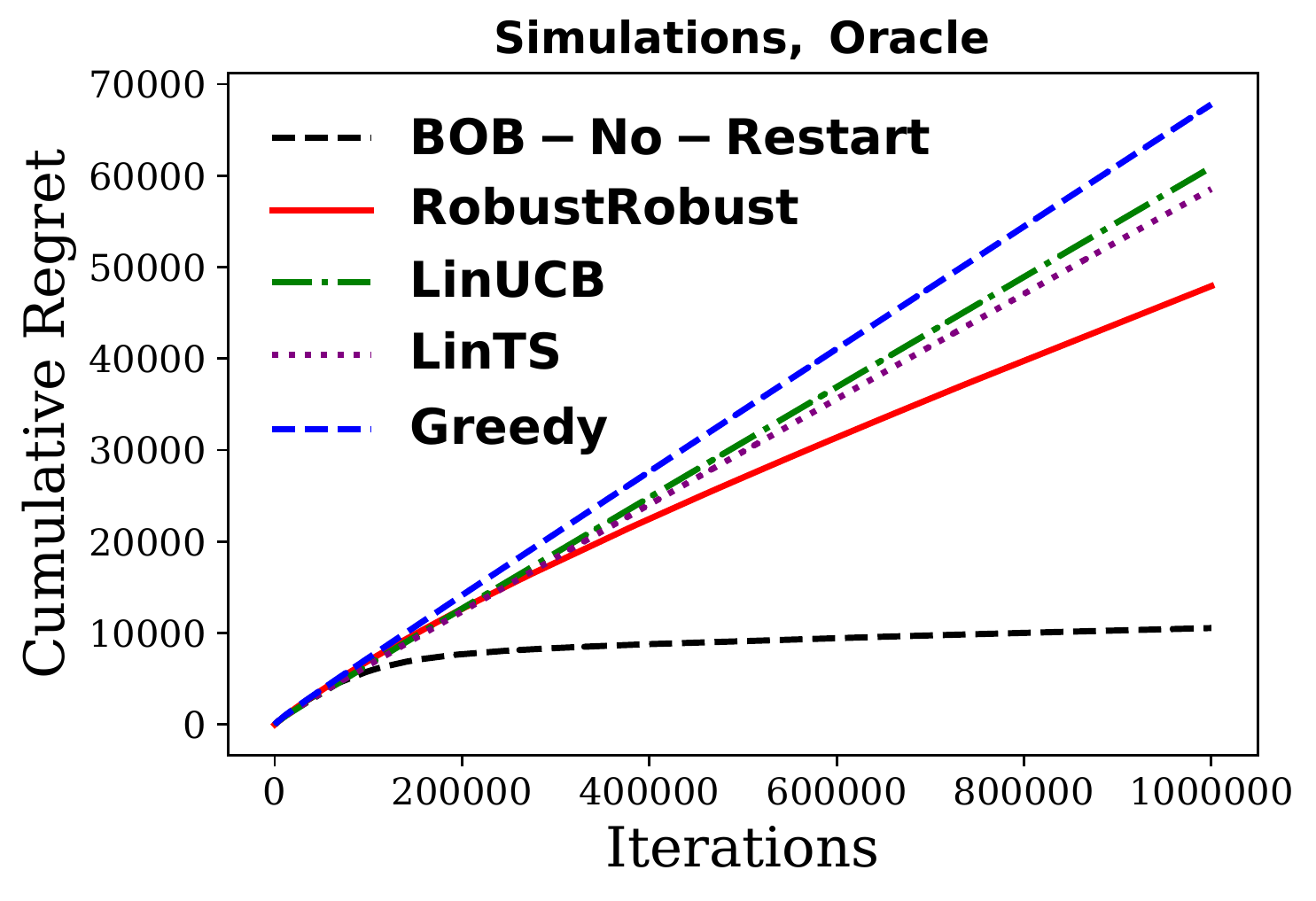}\includegraphics[width=0.33\textwidth]{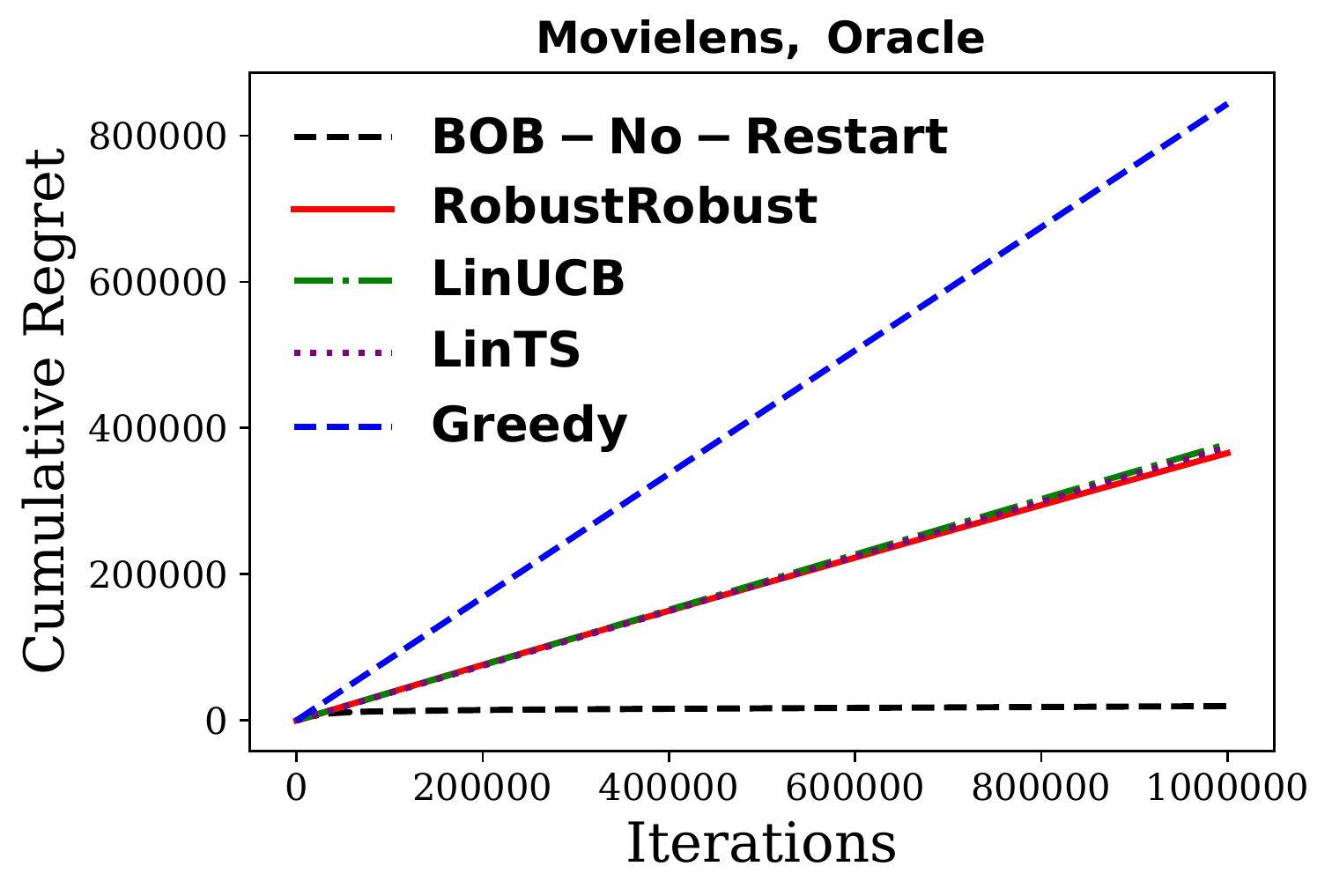}\includegraphics[width=0.33\textwidth]{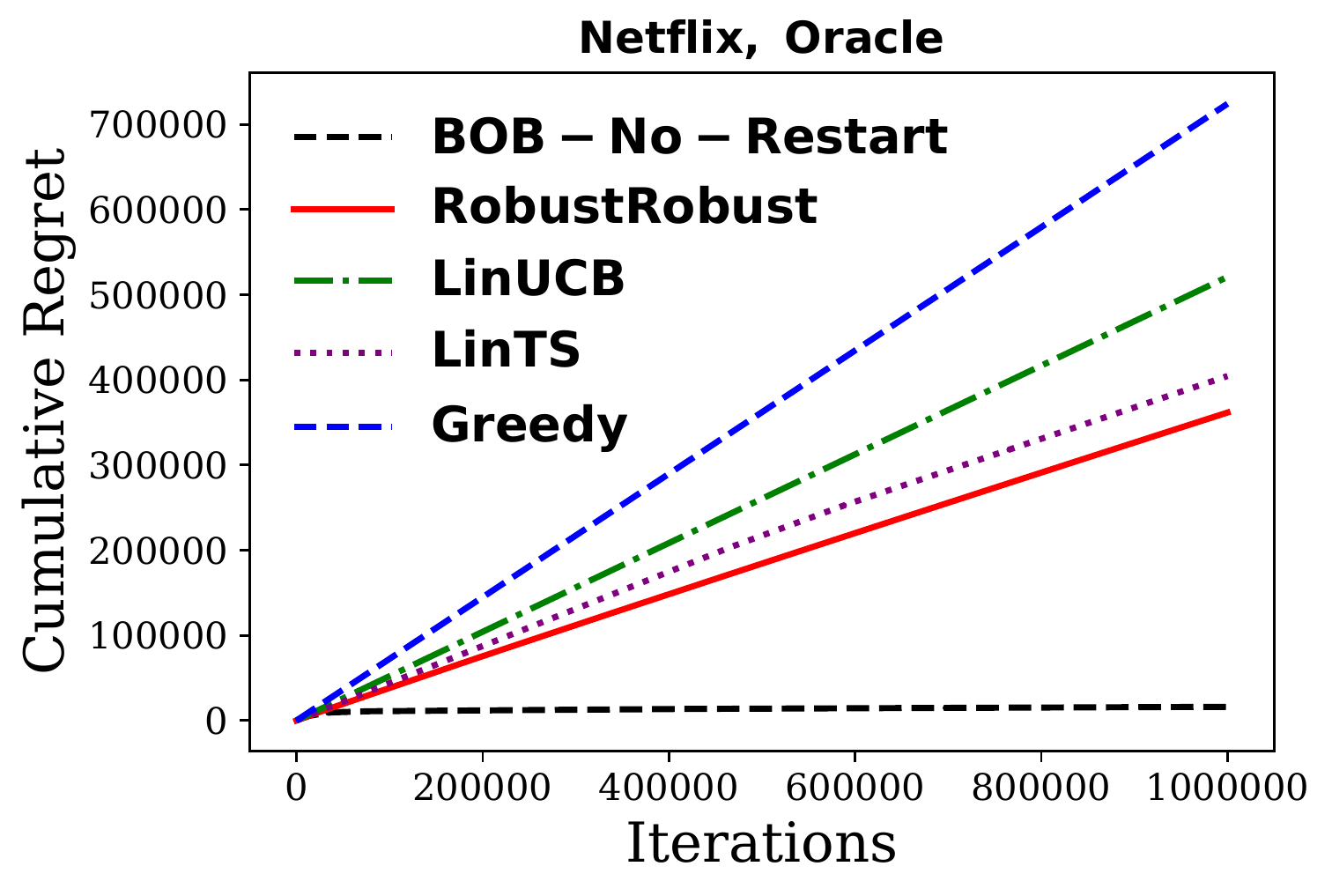}

    \caption{Comparison of LinUCB, LinTS, Greedy, RobustBandit and BOB-No-Restart algorithms under adversarial attacks on rewards and context.}
    \label{plot}
\end{figure*}

\section{EXPERIMENTAL RESULTS}\label{exp_robust}
In this section, we show by experiments on both synthetic and real datasets that our proposed algorithm is robust to adversarial corruptions. To the best of our knowledge, there is no existing robust algorithm that works for stochastic contextual bandit with changing feature vectors in the general environment. % Our work is also the first algorithm that can deal with attacks on context. 
We compare our proposed algorithm with LinUCB, LinTS and Greedy algorithm. % Since Greedy algorithm performs poorly in the real datasets and for better visualizations, we do not include the results of Greedy algorithm in the plots and defer the plots with Greedy algorithm to Appendix. 
% {\color{blue}and defer them to Appendix ??.}
% and show that our algorithm significantly outperforms them under adversarial attacks.

For all the experiments below, we set the number of total rounds $T=10^6$. At each round, there are $K=20$ arms and the dimension of features is $d=10$. We set the total attack budget as $C=100$. We choose the hyper-parameters exactly as Theorem \ref{final_thm} suggests. Suggested by the explicit formulae of regret bounds combined from Theorem \ref{reward} and Lemma \ref{thm_q2}, we set epoch length $H = \frac{\beta_T \gamma_{T+1}}{\sqrt{e-1}} \sqrt{T}$. 
Similar to \cite{garcelon2020adversarial}, we set error probability $\delta=0.01$ and regularization parameter $\lambda = 0.1$. % similarly as in \cite{garcelon2020adversarial}.  
If arm $a$ is pulled at round $t$, the true sample reward is drawn from $\mathcal{N} (x_{t,a}^T\theta, 0.01)$. An appropriate attack (See Section \ref{type_attack} for details of attacks) is assigned to the reward or the context and only the corrupted observation is shown to the player. Results reported below are averaged over $10$ independently repeated experiments. 
We run the experiments on three datasets listed below:

\textbf{1). Simulation}: We simulate all the feature vectors $x_{t,a}$ and the model parameter $\theta$ from $\text{Uniform} (-\frac{1}{\sqrt{d}}, \frac{1}{\sqrt{d}})$ to make sure $|x_{t,a}^T \theta| \leq 1$. Mean rewards $\mu_{t,a} = x_{t,a}^T \theta$ are then transformed by $\mu_{t,a} \gets \frac{x_{t,a}^T \theta +1}{2}$ to make sure the final mean rewards $\mu_{t,a}\in [0,1]$.

\textbf{2). Movielens 100K dataset}: This dataset has 100K ratings of 943 users on 1,682 movies. We use LIBPMF \citep{hfy12a,yu2014parallel} to perform matrix factorization with $d=10$ on the ratings data and get the features matrices for both users and movies. The movies are then treated as arms and for each repeated experiment, we randomly select $K=20$ movies. The model parameter $\theta$ is defined as the average of randomly selected $100$ users' feature vectors. The mean reward is then defined as $\mu_{t,a} = x_{t,a}^T \theta$ and shifted to $[0,1]$.

\textbf{3). Netflix dataset}: Netflix dataset contains ratings on 17,770 items from 480,189 users. The preprocessing of this data is almost the same as the Movielens dataset, except now that we use only the items having more than 10,000 ratings and this leaves us with $2042$ items. % Similarly, we randomly choose $K=20$ items for each repeated experiment and use the average of $100$ users' features as the model parameter. % The mean rewards are shifted to $[0,1]$.

\subsection{Type of Attacks}\label{type_attack}
% There are in total seven attacks we consider in our experiments. % : Garcelon topN $(N=K-1)$, Garcelon topN $(N=0.5K)$, Oracle topN $(N=K-1)$, Oracle topN $(N=0.5K)$, Flip $\theta$, Context topN $(N=K-1)$ and Context topN $(N=0.5K)$. 
Details of the attacks considered in the experiments are presented below. For all attacks, the attacker cannot attack anymore if the attack budget $C$ is exhausted. % We run the experiments under the following attacks. 

\textbf{1). Garcelon}: This attack is based on the reward attack in \cite{garcelon2020adversarial}.
Whenever the pulled arm is not among the attacker's target arms, the attacker modifies the reward into a random noise $\mathcal{N} (0, 0.01)$.

\textbf{2). Oracle}: This attack is a simple modification from \cite{jun2018adversarial}. It pushes the reward of an arm to some margin $\epsilon_0$ below the reward of the worst arm and does not attack if such condition is already attained. We use $\epsilon_0=0.01$ in the experiments. % To be more specific, it draws an attacked reward $\tilde r_t \sim \mathcal{N} (\min_a \mu_{t,a} - \epsilon_0, 0.01)$ and reveals the reward $\min (r_t, \tilde r_t)$ to the player.

For Garcelon and Oracle attacks in the above, the attacker attacks the context or rewards only when the pulled arm is among the top $N = 0.5K$ arms. This means that the attacker is malicious in the sense that it will attack when the pulled arm is among the best half arms.

\subsection{BOB-No-Restart: An empirically good algorithm based on the proposed method}

Although the restarting technique gives theoretically sub-linear regret bounds and it is necessary for the analysis of EXP3 layer to work, we found that restarting LinUCB usually abandons useful information in the past epochs. Empirically, a two-layer bandit structure without restarting performs exceptionally well (See Figure \ref{plot}) and we call this algorithm Bandit-Over-Bandit-No-Restart (BOB-No-Restart). To be more specific, BOB-No-Restart still has two layers and the top layer still guides the selection of $C^\prime$ to help enlarge the exploration parameter. However, instead of choosing $C^\prime$ every epoch, the EXP3 layer now chooses $C^\prime$ every round using the immediate observed reward from last round. The chosen $C^\prime$ now serves directly to the bottom LinUCB layer every round. We present more details of BOB-No-Restart in Algorithm \ref{bob-no-restart} in Appendix.

\subsection{Results}
We show the plots of averaged cumulative regret against number of rounds in Figure \ref{plot}. From left to right in Figure \ref{plot}, the plots are for simulations, Movielens and Netflix respectively. From the 1st row to the 2nd row, the attack types are Garcelon and Oracle respectively. From the plots, we can see that our proposed RobustBandit algorithm consistently outperforms LinUCB, LinTS and Greedy algorithms under adversarial attacks and improves the robustness. Moreover, BOB-No-Restart algorithm constantly improves robustness significantly compared with all algorithms. 

\section{CONCLUSION and FUTURE WORK}
In this work, we study the robustness of stochastic linear contextual bandit problems. We propose a first robust bandit algorithm for the contextual setting under adversarial attacks which can achieve sub-linear regrets. Our algorithm is also the first one that can deal with attacked context. Our algorithm does not need to know the attack budget or format in order to make decisions under finite-arm problems, while the attacker can be omniscient and fully adaptive after observing the player's decisions. 
The proposed algorithm has regret upper bound $\tilde O(d^{\frac{1}{2}} (C+1) T^{\frac{3}{4}})  + \tilde O( d^{\frac{3}{2}} C T^{\frac{1}{4}} )$. Extensive experiments show that our algorithm improves the robustness of bandit problems under attacks.

\paragraph{Future Work:} 
% There are several interesting future directions. 
First of all, as proved by \cite{garcelon2020adversarial}, if an algorithm has $f(T)$ regret under non-corrupted cases, the algorithm must suffer from linear regret when $C = \Omega (f(T))$. This shows that the regret lower bound does not solely relies on $C$ or $T$. Instead, the regret lower bound relies on interactions of both $C$ and $T$. 
We think a possible future direction is to derive the optimal balance between $C$ and $T$ in regret lower bounds for linear contextual bandit algorithm under attacks. Secondly, an optimal algorithm that matches this lower bound is of substantial interest. Lastly, the BOB-No-Restart algorithm has exceptionally good performance in experiments although it does not have theoretical guarantees. It is also interesting to investigate the theoretical behavior of this algorithm.

\bibliographystyle{plainnat}
\bibliography{main}

\onecolumn
\section{SUPPLEMENTARY MATERIAL}
\subsection{Proof of Lemma \ref{concen_reward}}
\begin{proof}
We prove the result holds for attack of rewards and context respectively below. 

\textbf{Attack rewards}: Note that in this case, $\tilde X_t = X_t$ and $\tilde V_t = V_t$, so we have from Equation \ref{1st_decomp} that
\begin{align}
\|\hat\theta_t - \theta\|_{V_t} \leq \|\hat\theta_t^0 - \theta\|_{V_t} + \|V_t^{-1} \sum_{s=1}^{t-1} c_s X_s \|_{V_t}
\end{align}
From the result in \cite{abbasi2011improved}, $\|\hat\theta_t^0 - \theta\|_{V_t} \leq \beta_t$ with probability at least $1-\delta$. For the second term in the above inequality,
\begin{align*}
    \|V_t^{-1} \sum_{s=1}^{t-1} c_s X_s \|_{V_t} \leq \sum_{s=1}^{t-1}  |c_s|\sqrt{X_s^T V_t^{-1} X_s} 
    \leq \sqrt{\sum_{s=1}^{t-1} |c_s|^2} \sqrt{\sum_{s=1}^{t-1} X_s^T V_s^{-1} X_s} 
    \leq \gamma_t C.
\end{align*}

\textbf{Attack context}:
Denote $c_s^* = \tilde X_s - X_s$. Then because of the bounded attack budget, $\sum_{t=1}^T \|c_s^*\| \leq C$.
For the ridge regression estimator under corrupted context, note that $\tilde Y_t = Y_t$ in this case, we have
\begin{align*}
    & \hat\theta_t - \theta = \tilde V_t^{-1} \sum_{s=1}^{t-1} \tilde X_s Y_s - \theta \\
    & = \tilde V_t^{-1} \sum_{s=1}^{t-1} (X_s+c_s^*) (X_s^T\theta + \epsilon_s) - \tilde V_t^{-1} (\lambda I_d + \sum_{s=1}^{t-1} \tilde X_s \tilde X_s^T ) \theta \\
    & = - \tilde V_t^{-1} \sum_{s=1}^{t-1} (X_s+c_s^*) {c_s^*}^T \theta + \tilde V_t^{-1} \sum_{s=1}^{t-1} \tilde X_s \epsilon_s  - \lambda \tilde V_t^{-1}  \theta \\
\end{align*}
For the first term in the above,
\begin{align*}
   &  \|\tilde V_t^{-1} \sum_{s=1}^{t-1} (X_s+c_s^*) {c_s^*}^T \theta\|_{\tilde V_t}
   = \| \sum_{s=1}^{t-1} \tilde X_s {c_s^*}^T \theta\|_{\tilde V_t^{-1}} \\
   & \leq \sum_{s=1}^{t-1} \sqrt{({c_s^*}^T \theta)^2 \tilde X_s^T \tilde V_t^{-1} \tilde X_s} 
   \leq \sum_{s=1}^{t-1} \|c_s^*\| \|\theta\| \sqrt{ \tilde X_s^T \tilde V_t^{-1} \tilde X_s}  \\
   & \leq \sqrt{\sum_{s=1}^{t-1} \|c_s^*\|^2} \sqrt{\sum_{s=1}^{t-1} \tilde X_s^T \tilde V_t^{-1} \tilde X_s} 
   \leq C \gamma_t.
\end{align*}
For the second term, $\| \tilde V_t^{-1} \sum_{s=1}^{t-1} \tilde X_s \epsilon_s \|_{\tilde V_t} = \| \sum_{s=1}^{t-1} \tilde X_s \epsilon_s \|_{\tilde V_t^{-1}}$. Using Theorem 1 in \cite{abbasi2011improved}, we have with probability at least $1-\delta$,
\begin{equation*}
    \| \tilde V_t^{-1} \sum_{s=1}^{t-1} \tilde X_s \epsilon_s \|_{\tilde V_t} 
    \leq \sqrt{2\sigma^2 \log \left( \frac{\det(\tilde V_t)^{1/2} \det(\lambda I_d)^{-1/2} }{\delta} \right)}.
\end{equation*}
Since $\text{tr}(\tilde V_t) = \lambda d + \sum_{s=1}^{t-1} \|\tilde X_s\|^2 \leq d\lambda + t$ and
$\det(\tilde V_t) \leq (\frac{1}{d}\text{tr}(\tilde V_t))^{d}$, we have $\| \tilde V_t^{-1} \sum_{s=1}^{t-1} \tilde X_s \epsilon_s \|_{\tilde V_t} \leq \sigma \sqrt{d\log \left( \frac{1+t/\lambda}{\delta}\right)}$.

For the third term, $\|\lambda \tilde V_t^{-1}  \theta\|_{\tilde V_t} \leq \sqrt{\lambda}$. So $\|\hat\theta_t - \theta\|_{\tilde V_t} \leq \beta_t + \gamma_t C$.

In conclusion, no matter it is under the attack of rewards or context, $\|\hat\theta_t - \theta\|_{\tilde V_t} \leq \beta_t + \gamma_t C$. 
The lemma follows by noticing that for all $x\in \mathbb{R}^d$,
\begin{equation*}
    | x^T (\hat\theta_t - \theta) | \leq  \|\hat\theta_t - \theta\|_{\tilde V_t} \|x\|_{\tilde V_t^{-1}}.
\end{equation*}
\end{proof}

\subsection{Proof of Theorem \ref{reward}}
\begin{proof}
Denote $\tilde r_t = (\tilde x_{t,*}-\tilde X_t)^T \theta$, we first bound $\tilde r_t$ and then relate this to the true single-round regret $r_t = (x_{t,*} - X_t)^T \theta$. 
According to Lemma \ref{concen_reward},
\begin{align*}
    & \tilde r_t = (\tilde x_{t,*} - \tilde X_t)^T \theta 
    \leq (\tilde x_{t,*} - \tilde X_t)^T \hat\theta_t + (\beta_t + \gamma_t C) \left(\|\tilde x_{t,*}\|_{\tilde V_t^{-1}} + \|\tilde X_t\|_{\tilde V_t^{-1}} \right) \\
    & \leq  (\beta_t + \gamma_t C^\prime)  \left( \|\tilde X_t\|_{\tilde V_t^{-1}} - \|\tilde x_{t,*}\|_{\tilde V_t^{-1}} \right) + (\beta_t + \gamma_t C)  \left(\|\tilde x_{t,*}\|_{\tilde V_t^{-1}} + \|\tilde X_t\|_{\tilde V_t^{-1}} \right)  \\
    & \leq  [2\beta_t + \gamma_t (C^\prime+C)] \|\tilde X_t\|_{\tilde V_t^{-1}}  + \gamma_t (C - C^\prime ) \|\tilde x_{t,*}\|_{\tilde V_t^{-1}}   \\
    & \leq (2\beta_t + \gamma_t (C+C^\prime)) \|\tilde X_t\|_{\tilde V_t^{-1}} \quad \text{ since } C\leq C^\prime
\end{align*}
So 
\begin{align*}
    \tilde R(T) & := \sum_{t=1}^T \tilde r_t 
    \leq (2\beta_T +  \gamma_T (C+C^\prime))\sum_{t=1}^T \|\tilde X_t\|_{\tilde V_t^{-1}} \\
    & \leq (2\beta_T +  \gamma_T (C+C^\prime)) \gamma_{T+1} \sqrt{T}.
\end{align*}

Under attack of rewards, $\tilde r_t = r_t$ since $\tilde x_{t,a} = x_{t,a}$, so $R(T) = \tilde R(T)$. While under attack of context, 
\begin{align*}
    r_t - \tilde r_t & = [(x_{t,*}-\tilde x_{t,*} - (X_t-\tilde X_t)]^T \theta \\
    & \leq (\|x_{t,*}-\tilde x_{t,*}\| + \|X_t-\tilde X_t\|) \|\theta\| \leq |c_t|
\end{align*}
So $R(T) = \sum_{t=1}^T r_t \leq \tilde R(T) + C$ under attack of context. In conclusion, with probability at least $1-\delta$, 
\begin{align}
    R(T) & \leq (2\beta_T +  \gamma_T (C+C^\prime)) \gamma_{T+1} \sqrt{T} \quad \quad \quad \text{ (attack rewards) }  \\
    R(T) & \leq (2\beta_T +  \gamma_T (C+C^\prime)) \gamma_{T+1} \sqrt{T} + C \quad \text{ (attack context)}
\end{align}
From Lemma 11 of \cite{abbasi2011improved}, we have that when $\lambda \geq \max(\max \|\tilde X_t\|^2, 1)$, 
\begin{equation}\label{bound_gamma}
    \gamma_{t}^2 \leq 2\log \frac{\det (\tilde V_t)}{\det (\lambda I_d)} \leq 2d \log \frac{\frac{1}{d} \text{tr}(\tilde V_t)}{\lambda} \leq 2d \log (1+\frac{t-1}{d \lambda}).
\end{equation}
Therefore, $\beta_T, \gamma_T \leq O(\sqrt{d\log T})$, which ends the proof.
\end{proof}

\subsection{Proof of Lemma \ref{thm_q2}}
\begin{proof}
From Corollary 3.2 in \cite{auer2002nonstochastic}, the regret of EXP3 algorithm is upper bounded by $2\sqrt{e-1} Q \sqrt{ L K^\prime \log K^\prime}$, where Q is the maximum absolute sum of rewards in any epoch, $L$ is the number of rounds and $K^\prime$ is the number of arms. In our case, $Q\leq H$, $L = \lceil \frac{T}{H} \rceil$ and $K^\prime = |J|$. So 
we have 
\begin{align*}
  \text{Quantity (B)} \leq 2\sqrt{e-1} H \sqrt{\lceil\frac{T}{H}\rceil |J| \log |J|}
  \leq \tilde O(\sqrt{TH}).
\end{align*}
\end{proof}

\subsection{LinUCB algorithm for solving non-corrupted linear contextual bandit problem}
% In this section, we present the classic LinUCB \citep{li2010contextual,abbasi2011improved} algorithm below.
\begin{algorithm}[H] 
\caption{LinUCB Algorithm}
\label{linucb}
\textbf{Input}: $T$, $\lambda$, sub-Gaussian parameter $\sigma$, error probability $\delta$. 
\begin{algorithmic}[1]
\STATE Initialize $V_t = \lambda I_d$, $\hat\theta_t = 0$ and $b_t = 0$.
        \FOR{$t = 1$ {\bfseries to} $T$}
        \STATE Observe context $x_{t,a}$.
        \STATE Calculate exploration parameter $\beta_t = \sigma\sqrt{d \log \left( \frac{1+t/\lambda}{\delta} \right)} + \sqrt{\lambda}$.
        \STATE Pull arm $a_t = \argmax_a x_{t,a}^T \hat\theta_t + \beta_t \|x_{t,a}\|_{{V_t}^{-1}}$.
        \STATE Observe reward $Y_t$.
        \STATE Update LinUCB components
        %\begin{equation*}
            $V_{t+1} = V_t + X_t X_t^T$, $b_{t+1} = b_t + X_t Y_t$ and $\hat\theta_{t+1} = V_{t+1}^{-1} b_{t+1}$
        % \end{equation*}
    \ENDFOR
\end{algorithmic}
\end{algorithm}

\subsection{BOB-No-Restart Algorithm}

\begin{algorithm}[htb] 
\caption{BOB-No-Restart Algorithm}
\label{bob-no-restart}
\textbf{Input}: $T$, $\beta_t$ as in Equation \ref{beta}.
\begin{algorithmic}[1]
    \STATE Initialize $\hat\theta_t = 0$, $\tilde V_t = \lambda I_d$ and candidate attack budget set $J= \{0\}\cup\{2^j\}_{j=0}^{\lceil \log_2 2KT\rceil}$. 
    \STATE Initialize exponential weights $w_j(1) = 1$ for $j=1,\dots,|J|$. 
    \STATE Initialize the exploration parameter for EXP3 as $\alpha = \min\left\{ 1, \sqrt{\frac{|J| \log |J|}{(e-1)T}} \right\}$.
    % \STATE Initialize 
    \FOR{$t = 1$ {\bfseries to} $T$}
        \STATE Observe (attacked) contextual feature vectors $\tilde x_{t,a}$.
        \STATE Update probability distribution for pulling candidates in $J$ as
        \begin{equation*}
            p_{j}(t) = \frac{\alpha}{|J|} + (1-\alpha) \frac{w_j(t)}{\sum_{i=1}^{|J|} w_i(t)}, \forall j=1,\dots, |J|.
        \end{equation*}
        \STATE Draw $j_t \gets j \in [|J|]$ with probability $p_j(t)$. % {\color{red}(can we just say sample $j_t$ from $[J]$ with probability $p_j(t)$? It is confusing if we have both $i$ and $j$ here.)}
        \STATE Pull arm according to the following equation
        \begin{align*}
            a_t & = \argmax_{a=1,\dots,K} \tilde x_{t,a}^T \hat\theta_t + (\beta_t + \gamma_t J_{j_t}) \|\tilde x_{t,a}\|_{\tilde V_t^{-1}} 
        \end{align*}
        \STATE Observe reward $\tilde Y_t$. %  (attack rewards) or $Y_t$ (attack context).
        \STATE Update LinUCB components, i.e., $\tilde V_{t+1} = \tilde V_t + \tilde X_t \tilde X_t^T$ and $\hat\theta_{t+1} = \tilde V_{t+1}^{-1} \sum_{s=1}^t \tilde X_s \tilde Y_s$.
        \STATE Update EXP3 components: $\hat y_t(j) \gets 0$ for all $j\neq j_t$, $\hat y_t(j) = \tilde Y_t / p_{j} (t)$ if $j=j_t$ and % {\color{red}(replace $i_t$ by $j_t$ in the following? )}
        \begin{align*}
            % \hat y_t(j)& = \tilde Y_t / p_{j} (t) \quad \text\\ % \tag{Attack rewards} \\
            % \hat y_t(j) & = Y_t / p_{j} (t) \tag{Attack context} \\ 
            w_j(t+1) & = w_j(t) \times \text{exp} \left( \frac{\alpha}{|J|} \hat y_t(j)\right).
        \end{align*}
    \ENDFOR
\end{algorithmic}
\end{algorithm}
\end{document}